\documentclass{article}

\usepackage{times}
\usepackage{graphicx} 
\usepackage{subfigure} 

\usepackage{natbib}

\usepackage{algorithm}
\usepackage{algorithmic}

\usepackage{hyperref}

\usepackage{amssymb}
\usepackage{verbatim}
\usepackage{amsmath,amsthm,amsfonts,amssymb,amscd,hyperref}
\usepackage{framed}

\usepackage[accepted]{icml2017}

\icmltitlerunning{Mixing Complexity and Neural Networks}

\newtheorem{lemma}{Lemma}
\newtheorem{theorem}[lemma]{Theorem}
\newtheorem{corollary}[lemma]{Corollary}
\newtheorem{definition}[lemma]{Definition}

\newtheorem{claim}[lemma]{Claim}

\newcommand{\E}{\mathbb{E}}
\newcommand{\cH}{\mathcal{H}}
\newcommand{\cD}{\mathcal{D}}

\newcommand{\cX}{\mathcal{X}}

\begin{document} 

\twocolumn[
\icmltitle{Mixing Complexity and its Applications to Neural Networks}




\begin{icmlauthorlist}
\icmlauthor{Michal Moshkovitz}{h}
\icmlauthor{Naftali Tishby}{h}
\end{icmlauthorlist}

\icmlaffiliation{h}{The Hebrew University of Jerusalem, Israel}

\icmlcorrespondingauthor{Michal Moshkovitz}{michal.moshkovitz@mail.huji.ac.il}
\icmlcorrespondingauthor{Naftali Tishby}{tishby@cs.huji.ac.il}


\vskip 0.3in
]



\printAffiliationsAndNotice{}  

\begin{abstract} 
We suggest analyzing neural networks through the prism of space constraints. 
We observe that most training algorithms applied in practice use bounded memory, which enables us to use 
a new notion introduced in the study of space-time tradeoffs that we call \emph{mixing complexity}.
This notion was devised in order to measure the (in)ability to learn using a bounded-memory algorithm.
In this paper we describe how we use mixing complexity to obtain new results on what can and cannot be learned using neural networks.
\end{abstract} 

\section{Introduction}
Understanding neural network learning is an active research area in machine learning and neuroscience \cite{shamir16,safran16,eldan16,daniely16,maithra16,arora14,livni14,tishby15,kadmon16}.
In this paper we view this problem through a different lens --- that of space constraints. 
We observe that learning with neural networks, either artificial or biological, is almost always done using a bounded-memory algorithm. 
In the setting of machine learning, artificial neural networks most often use the Stochastic Gradient Descent (SGD) algorithm, one example of a bounded-memory algorithm. 
In neuroscience, biological neural networks, i.e., the nervous system, inherently perform a bounded-memory computation under the accepted assumption that learning should be \emph{biologically plausible}.
This places the problem of learning with neural networks in the framework of bounded-memory learning. 

In recent years, several works have shown that under memory constraints numerous examples are needed in order to learn certain hypothesis classes \cite{shamir14,raz16,kol16,moshkovitz17,raz17}. 
Expanding on the work of \cite{moshkovitz17} we define a new complexity measure, \emph{mixing complexity}, in order to evaluate the difficulty of learning a class under memory constraints.
For a class $\cH$ we denote its mixing complexity by $MC(\cH)$; the larger it is, the more ``complex'' $\cH$ is. 
Roughly speaking, $MC(\cH)$ assesses the closeness of $\cH$ to a random class. 
It does so by viewing $\cH$ as a bipartite graph and comparing its edge distribution to that of a truly random graph.

In this paper we explore mixing complexity and its applications to neural networks.
Our first application shows that any class $\cH$ with mixing complexity $MC(\cH)=\Omega(\sqrt{|\cH|})$ cannot be learned by neural networks.
One implication is that most classes cannot be learned by neural networks.

The above might seem to contradict the fact that, empirically, neural networks do indeed learn \cite{krizhevsky12,lecun15}.
To bridge this gap, we suggest that ``natural'' hypothesis classes have certain ``symmetries''.
We formalize this notion and prove that such natural classes have high mixing complexity.
Thus, mixing complexity sheds a new light on our understanding of what can and cannot be learned using neural networks. 

In addition we discuss an application of mixing complexity as a response to a question raised by \cite{zhang17}.
They showed that classification of natural images by a specific neural network achieves a small generalization error.
However, for random labels the same neural network suffers from a large generalization error.
They showed that classical measures fail to explain these findings, and left as an open problem the task of finding a more suitable complexity measure. 
We show that mixing complexity does in fact distinguish between natural images and random classes. 

Given the usefulness of mixing complexity, as demonstrated by the above applications, we prove that it has several desirable properties. 
First, we prove that any class $\cH$ with mixing complexity $MC(\cH)=\Omega(\sqrt{|\cH|})$ has a VC-dimension $\Omega(\log|\cH|)$. 
Since the VC-dimension of every class $\cH$ is at most $\log_2|\cH|$, we get that such classes are the hardest to learn (up to a constant factor) without memory constraints. 
Furthermore, we prove that mixing complexity is robust under small perturbations.
That is, if a small number of labels is changed,  the mixing complexity is approximately unchanged.

\subsection{Paper overview}
In Section~\ref{sec:preliminaries} we briefly review the definitions of learning, the VC-dimension, and bounded-memory algorithms. 
In Section~\ref{sec:nn_bounded_memory} we show that artificial and biological neural networks each use a bounded-memory algorithm. 
In Section~\ref{sec:mixing_graphs} we define mixing complexity and exemplify it.
In Section~\ref{sec:unlearnability_of_mixing} we restate the main theorem proved in \cite{moshkovitz17}. It claims that bounded-memory algorithm cannot learn hypothesis classes that are ``mixing''.
In Section~\ref{sec:when_can_learn} we explore a set of natural classes and prove they are not mixing; i.e., a bounded-memory algorithm may be able to learn these classes.
In Section~\ref{sec:google} we describe the work done in  \cite{zhang17} and its connection to mixing complexity.
In Section~\ref{sec:perturbation_robustness} we prove some desirable properties of mixing complexity.
Section~\ref{sec:conclusions} summarizes the results and leaves some open problems  for future work. 

\section{Preliminaries}\label{sec:preliminaries}
\subsection{Learning}
Learning is the process of converting experience into expertise. 
A learner receives labeled examples $(x,b)\in\cX\times\{0,1\}$ one after another as experience and after enough examples the learner outputs a \emph{hypothesis} $h:\cX\rightarrow\{0,1\}$. 
The goal is to return $h$ that minimizes the \emph{test error}, which is defined as the probability to return a different answer than the true underlying hypothesis $f$ $$L_{(D,f)}(h)=\Pr_{x\sim \cD}[h(x)\neq f(x)]$$ 
The examples in the learning process are drawn independently from some unknown distribution $\cD$ and an underlying hypothesis $f$.

\begin{definition}[PAC learnable, \cite{valiant84}]
A hypothesis class $\cH$ is \emph{PAC learnable} if there exists a function $m_{\cH}:(0,1)^2\rightarrow\mathbb{N}$ and a learning algorithm $A$ with the following property: For every $\epsilon,\delta\in(0,1)$, for every distribution $\cD$ over $\cX$, and for every underlying hypothesis $f\colon\cX\to\{0,1\}$, when running the learning algorithm on $m\geq m_{\cH}(\epsilon,\delta)$ i.i.d examples generated by $\cD$ and labeled by $f$,  the algorithm returns a hypothesis $f$ such that, with a probability of at least $1-\delta$ (over the choice of the examples), $$L_{(D,f)}(h)\leq\epsilon.$$
\end{definition}
Given a series of examples $S=\{(x_i,y_i)\}$ we define the \emph{training error} as $$L_S(h)=\frac{1}{|S|}\sum_{i=1}^{|S|}I_{h(x_i)\neq y_i},$$ where $I_P$  returns $1$ if the predicate $P$ is true, else $0.$ 
The difference between the training error and the test error is called the \emph{generalization error}.  

\subsection{The VC-dimension Complexity}\label{sub_sec:preliminaries_VC}
The VC-dimension is a complexity measure for hypothesis classes, as proven in the \emph{Fundamental Theorem of Statistical Learning}.
\begin{definition}[restriction]
Let $\cH$ be a hypothesis class for binary classification over the domain $\cX$.
Let $S=\{x_1,\ldots,x_{n}\}$ be a set of examples. 
The \emph{restriction} of $\cH$ into $S$ is the set $$\cH_S=\{(h(x_1),\ldots,h(x_n)) :  \; h\in \cH \}.$$
\end{definition}

\begin{definition}[shattering]
A hypothesis class $\cH$ \emph{shatters} a set $S\subseteq\cX$ if $|\cH_S|=2^{|S|}.$
\end{definition}

\begin{definition}[VC-dimension]
The VC-dimension of a hypotheses class $\cH$ is $$VCdim(\cH)=\sup\{|S|  : \; \cH\text{ shatters }S\}$$
\end{definition}

\begin{theorem}[Fundamental Theorem of Statistical Learning]
Let $\cH$ be a hypothesis class over the domain $\cX$, and let $d=VCdim(\cH)$. 
Then, there are absolute constants $C_1, C_2$ such that $\cH$ is PAC learnable with sample complexity
$$C_1\frac{d+\log\frac1{\delta}}{\epsilon}\leq m_{\cH}(\epsilon,\delta) \leq C_2\frac{d\log\frac1\epsilon+\log\frac1\delta}{\epsilon}.$$
\end{theorem}
For convenience, we henceforth set $\epsilon,\delta=1/4$. It is known that any class $\cH$ can be learned with $O(\log|\cH|)$ examples without memory constraints. 
If a bounded-memory learning algorithm must use at least $|\cH|^c$, for some constant $c>0$, we say that the class is \emph{unlearnable} with these memory constraints.

\subsection{Bounded-Memory Learning Algorithm}
One approach to designing a learning algorithm is to save all the examples received and return a hypothesis with a minimal training error.
Notice that this approach does not use a bounded memory. 
In this paper, however, we focus on bounded-memory algorithms.

A bounded-memory algorithm is a Turing machine with a bounded size tape $s$ with each cell in the tape being either $'0'$ or $'1'$.
It is useful to think of such an algorithm as a graph on $\Lambda=2^s$ vertices. 
Each vertex is one possible \emph{memory state}.
In each step, the algorithm is in one memory state.
When faced with a new example, the algorithm transition to another memory state. 
In the final step, the algorithm outputs a hypothesis that depends on the memory state where it ended up.

In this paper we investigate hypothesis classes that are unlearnable with a bounded-memory algorithm. 
We adhere to the \emph{realizability assumption}, which means that there is a hypothesis in the class with a test error equal to $0.$
Notice that using this assumption only strengthens the unlearnability result.  
Also note that this paper focuses on the statistical aspect (i.e., how many examples are needed to learn) and not the computational aspect (i.e., how much time is needed to learn). 
Notice that proving that bounded-memory algorithms must use many examples (the statistical aspect) immediately yields that the bounded-memory algorithms must run slowly (the computational aspect).

\section{Neural Networks and Bounded-Memory Algorithms}\label{sec:nn_bounded_memory}
In this section we discuss the connection between bounded-memory algorithms and neural networks in machine learning and neuroscience.
From the perspective of machine learning, neural networks define a hypothesis class. 
The algorithm that is almost always used to find a hypothesis from this class is a bounded-memory algorithm (i.e., the stochastic gradient descent algorithm), as explained below.  
From the perspective of neuroscience, we show that under the accepted assumptions, any computation made by the nervous system must be a bounded-memory algorithm. 

\subsection{Neural Networks and Machine Learning}\label{sub_sec:nn_and_ml}
Artificial neural networks have dramatically improved the state-of-the-art in many fields (see \cite{lecun15} and referenced therein).
In general, however, learning a neural network is NP-hard \cite{blum88}. 
This has led many researchers to attempt to understand the reasons for the astonishing success despite the proven hardness \cite{shamir16,safran16,eldan16,daniely16,maithra16,arora14,livni14}.

In this section we establish one property of the widely used algorithm for learning neural networks: it is a bounded-memory algorithm.
A feed-forward artificial neural network is composed of layers of neurons and directed edges between consecutive layers, this is known as the \emph{architecture} of the neural network. 
Each neuron computes the mapping $\sigma(w\cdot x + b),$ where $x$ is the input to the neuron, $w$ is a vector that represents the weight of each input, $b$ is a bias term, and $\sigma:\mathbb{R}\rightarrow\mathbb{R}$ is some activation function.

The stochastic gradient descent (SGD) method is a popular way to learn the weights of a neural network. 
When it gets a new example (or a small number of examples) it changes the current weights of the neural network, based on the appropriate gradient.
Thus it is a bounded-memory algorithm.

\subsection{Neural Networks and Neuroscience}\label{sub_sec:nn_and_neuroscience}
The nervous system is responsible for processing all the information an organism receives and acting accordingly.
It is composed of a large number of neurons that are connected to one another and together form a biological neural network.
Two neurons are connected through a synapse, a structure that permits the transmission  of an electrical or chemical signal from one neuron (called the pre-synaptic neuron) to another (called the post-synaptic neuron). 
Each pre-synaptic neuron can influence the post-synaptic neuron differently, depending on various biological parameters (e.g., the amount of neurotransmitter released into the synapse, the myelination).

In the standard model of a neural network, each synapse that connects a pre-synaptic neuron $i$ and a post-synaptic neuron $j$ is represented by some weight $w_{i,j}$ that scales the input from the pre-synaptic neuron activity $x_i$. 
The post-synaptic neuron computes the inner product of the multiple pre-synaptic neurons' activities and the synaptic weights $\langle w_i, x\rangle = \sum_i w_{i,j}x_i$. 
This inner product passes through a nonlinearity $\sigma$ called the activation function and the result $\sigma(\langle w_i, x\rangle)$ is transmitted to the post-synaptic neuron.

In the standard model, learning in the nervous system manifests itself by a change in the synaptic weight. 
Neuroscientists focus solely on changes that are \emph{biologically plausible}; i.e., those that fulfill some biological constraints. 
The most important constraint, for our purposes, is that the changes in weights are only a function of the current sensory input. 
This means that the brain computes a bounded-memory algorithm.

\section{Mixing Complexity}\label{sec:mixing_graphs}
\subsection{Hypotheses Graphs}
A hypothesis class can be viewed as a bipartite graph in which on one side there is a vertex for each hypothesis $h$ and on the other side there is a vertex for each example $x$, and there is an edge $(h,x)$ if and only if $h(x)=1$. 
To illustrate it, focus on the class $\cH_{th}$ of discrete threshold functions in $[0,1]$: 
\begin{itemize}
\item the examples are the numbers $$\cX=\left\{0=\frac{0}{|\cX|-1},\frac{1}{|\cX|-1},\ldots,\frac{|\cX|-1}{|\cX|-1}=1\right\}$$ 
\item the hypotheses correspond to the $|\cX|+1$ thresholds $b\in \left\{-1, \frac{1}{2(|\cX|-1)},\frac{3}{2(|\cX|-1)},\ldots,\frac{2|\cX|}{2(|\cX|-1)}\right\}$ and $h_b(x)=1$ if $x\leq b$ and $0$ otherwise. The number of hypotheses is equal to $|\cH_{th}|=|\cX|+1$.
\end{itemize}
The graph that corresponds to $\cH_{th}$ is a bipartite graph of size $(|\cX|+1)\times |\cX|$ and the edge  $(h_b,x)$ exists in the graph if $h_b(x)=1$, see Figure~\ref{fig:half_graph}.
Viewing the class as a graph enables us to examine some properties of the graph; e.g., the number of edges in this graph is $\frac{|\cX|(|\cX|+1)}{2}=\frac{|\cH_{th}||\cX|}{2}$ which is exactly half of the maximal number of edges $|\cH_{th}||\cX|$.

Another hypothesis class we consider is parities $\cH_{parity}:$
\begin{itemize}
\item the examples are all the binary points in $\{0,1\}^{\log_2|\cX|}$, where $|\cX|$ is a power of $2.$
\item the hypotheses correspond to all subsets $C\subseteq\{0,1\}^n$, except the empty one and $h_C(x)=\sum_{i\in C}x_i,$ where $x_i$ is the $i$ coordinate of $x$. The size of the hypothesis class is $|\cH_{parity}|=|\cX|-1$
\end{itemize}
The number of edges in the graph is equal to $|\cH_{parity}|\frac{|\cX|}{2}$ which is half of all the possible edges $|\cH_{parity}| {|\cX|}$. 

\subsection{Mixing Graphs}
In Section~\ref{sec:unlearnability_of_mixing} we will utilize this new view of hypothesis classes as graphs to deduce the unlearnability of some classes when the learning algorithm has bounded memory. 
This result will hold for classes with hypotheses graphs that are ``close'' to random. 
There are many ways to be close to random. One natural way is by using the edge-distribution, as discussed below.

For convenience we only consider hypothesis classes where on average over the hypotheses $h$, the number of examples $x$ with $h(x)=1$ is roughly the same as the number of examples with $h(x)=0$; i.e.,  $\E_{x,h}[h(x)=1]$ is close to $1/2$. 
Equivalently, out of all the possible numbers of edges in the graph, $|\cH|\cdot|\cX|$, there will be very close to $\frac{|\cH|\cdot|\cX|}{2}$ edges in the graph.
Stated differently, the expected number of edges between random subsets of vertices $S,T$ with $|S|=s, |T|=t$ is approximately $e(S,T)\approx \frac{st}{2}.$ 
In a random graph we expect that the number of edges between \emph{any} two subsets will be close to their average, up to the standard deviation $\sqrt{st}.$ More formally, 
\begin{definition}[$\mathrm{d}$-mixing]
A bipartite graph $G=(A,B,E)$ is $\mathrm{d}$-mixing if for any $T\subseteq A, S\subseteq B$ with $|S|=s, |T|=t$ it holds that $$\left|e(S,T) - \frac{st}{2}\right|\leq \mathrm{d}\sqrt{st}.$$
\end{definition}

We remark that the latter definition can be stated for graphs whose density, $\frac{e(A,B)}{|A||B|}$, is differs from half \cite{krivelevich06}\footnote{A more general definition of $\mathrm{d}$-mixing is: for any $T\subseteq A, S\subseteq B$ with $|S|=s, |T|=t$ it holds that $$\left|e(S,T) - \frac{st}{\frac{e(A,B)}{|A||B|}}\right|\leq \mathrm{d}\sqrt{st}.$$}.  

Note that the graph is closer to random as $\mathrm{d}$ gets smaller. 
We can find an immediate upper bound for $\mathrm{d}$ since we can bound the number of edges between any sets $S$ and $T$ by $0 \leq e(S,T)\leq st$, thus $$|e(S,T)-st/2|\leq st/2 = \frac{\sqrt{st}}{2}\sqrt{st}\leq\frac{\sqrt{|A||B|}}2\sqrt{st}.$$
Thus we get that $0\leq\mathrm{d}\leq\sqrt{|A||B|}$.  
Now we are ready to formally define \emph{mixing complexity}.
\begin{definition}[mixing complexity]
For hypothesis class $\cH$ over $\cX$ denote by $\mathrm{d}_{min}(\cH)$ the minimal value such that $\cH$ is $\mathrm{d}_{min}$-mixing. The mixing complexity of $\cH$ is $$MC(\cH)=\frac{\sqrt{|\cH||\cX|}}{\mathrm{d}_{min}(\cH)}.$$ 
\end{definition}
 We say that a class is \emph{mixing} if it is $\mathrm{d}$-mixing with $\mathrm{d}=O(\sqrt{\cX}),$ or equivalently if $MC(\cH)=\Omega(\sqrt{\cH}).$

Let us explore the mixing complexity of the hypothesis classes that we considered earlier.
We will prove that the class $\cH_{th}$ is $\Omega(\sqrt{|\cH_{th}||\cX|})$-mixing; i.e., $MC(\cH_{th})=O(1),$ which means that this class is mixing.
To show this take the first half of the hypothesis $T_0=\{h_b :\; 0\leq b<1/2\}$ with $|T_0|=t$ and the last half of the examples $S_0=\{x\; | 1/2<x\leq 1\}$ with $|S_0|=s$. 
By the definition of $\cH_{th}$ there are no edges between these two sets, i.e., $e(S_0,T_0)=0$, see Figure~\ref{fig:half_graph_mixing}.
However, we expect a great deal of edges $st/2=\Omega(|\cH_{th}||\cX|)$ between these two large sets. 
Hence, if $\cH_{th}$ is $\mathrm{d}$-mixing, then $$\mathrm{d}\geq\frac{\left|e(S_0,T_0)-\frac{st}2\right|}{\sqrt{st}}=\frac{\sqrt{st}}2=\Omega(\sqrt{|\cH_{th}||\cX|}).$$

On the other hand, the class $\cH_{parity}$ is very close to random: Lindsey's Lemma states that $MC(\cH_{parity})=\Omega(\sqrt{|\cH_{parity}|}).$ 

One might wonder what kind of classes are more abundant, the $\mathrm{d}$-mixing with small or large $\mathrm{d}$. 
Apparently almost all classes $\cH$ have small $\mathrm{d}=O(\sqrt{|\cX|})$, i.e., $MC(\cH)=\Omega(\sqrt{|\cH|})$ \cite{krivelevich06}. 
This fact can be proved using the Chernoff bound and the union bound, similar to the way one can prove that $O(\log|\cH|)$ examples are enough to learn any class $\cH$ without memory constraints. 

\begin{figure}[ht]
\vskip 0.2in
\begin{center}
\centerline{\includegraphics[ scale=0.7]{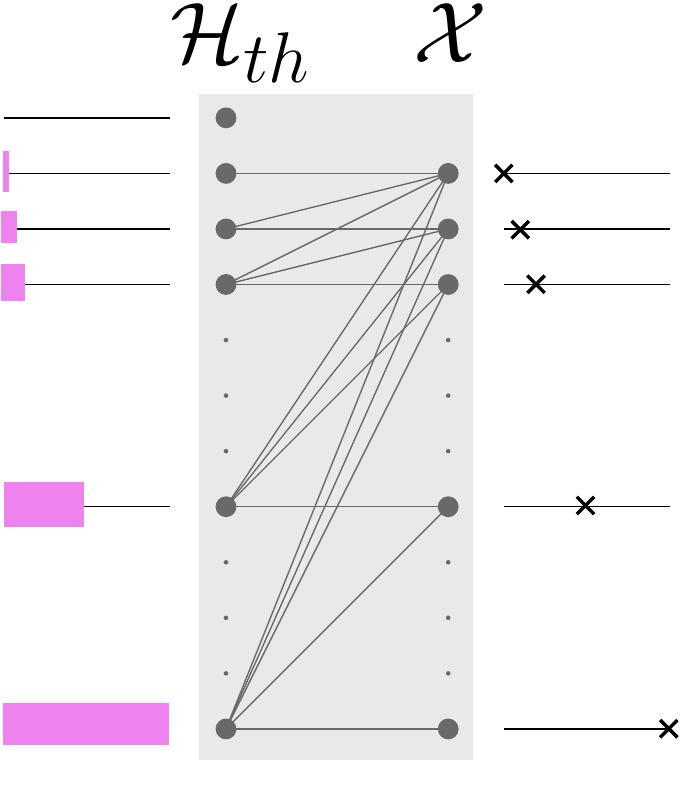}} 
\caption{The hypotheses graph of $\cH_{th}$ (middle). The right side consists of all the examples, each represented by a cross on the segment $[0,1].$ The left side consists of all the hypotheses, each returns $1$ on all inputs in the violet rectangle.}
\label{fig:half_graph}
\end{center}
\vskip -0.2in
\end{figure} 

\begin{figure}[ht]
\vskip 0.2in
\begin{center}
\centerline{\includegraphics[scale=0.7]{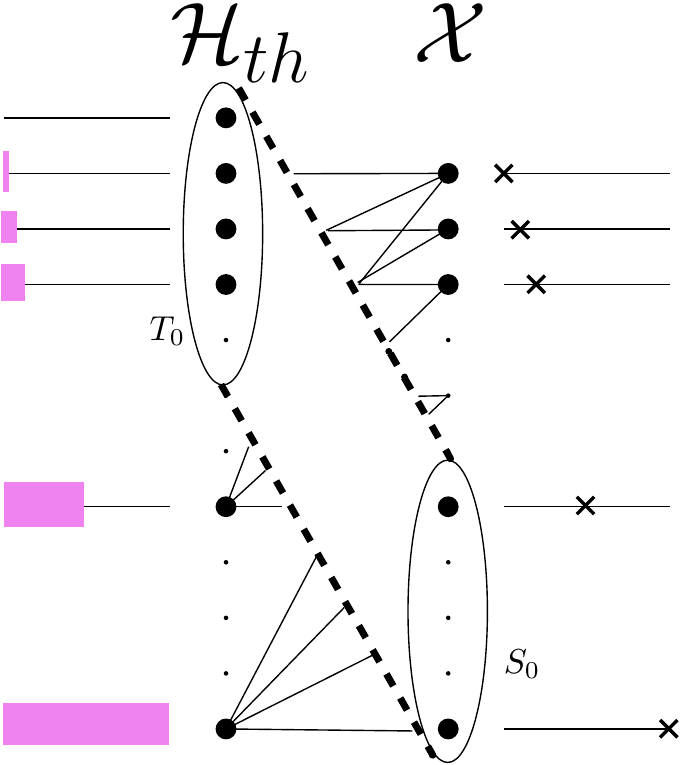}}
\caption{Schematic drawing of the proof that $\cH_{th}$ is not mixing. Focusing on the the sets $S_0,T_0$, we see that $e(S_0,T_0)=0.$}
\label{fig:half_graph_mixing}
\end{center}
\vskip -0.2in
\end{figure}

\section{Unlearnability of Mixing Classes with Bounded-Memory Algorithm}\label{sec:unlearnability_of_mixing}
In this section we restate the main theorem proved in \cite{moshkovitz17}.
This theorem shows that for any hypothesis class $\cH$ that is mixing, any learning algorithm with a bounded memory cannot learn $\cH.$
By \emph{cannot learn} we mean that the number of examples needed to learn the class is at least $|\cH|^c$, for some small constant $c>0.$
Recall that if the memory is not bounded, number of examples needed to learn $\cH$ is $O(\log|\cH|).$
The intuition behind this result is that if the memory is bounded, the learning algorithm must cope with the situation that many labeled examples $S$ will lead to the same memory state $m$.
For graphs that are mixing, most hypotheses $h$ are almost equally alike from the point of view of the memory $m.$  

We are interested in classes that are close enough to random such that $\mathrm{d}^2=O(|\cX|| \cH|^a),$ for some small enough constant $a\geq 0.$ 
For the parity class we described in the last section this holds since $\mathrm{d}^2=O(|\cX|).$
For such classes \cite{moshkovitz17} proved that a bounded-memory algorithm cannot learn this class. Specifically, 

\begin{theorem}\label{thm-maim}
Suppose that hypothesis class $\cH$ over domain $\cX$ is $\mathrm{d}$-mixing with $\mathrm{d}^2=|\cX||\cH|^{a}$ for some constant $a\in[0,1]$, $|\cH|$ is at least some constant, and 
$\left|\frac{e(\cH,\cX)}{|\cH|}-\frac{|\cX|}2\right|\leq\mathrm{d}\sqrt{\frac{|\cX|}{|\cH|}}$, then for any constant $s\in(0,1)$ there is a constant $s'>0$ such that any learning algorithm for $\cH$ that has at most
 $$|\cH|^{1.25-s-3a}$$ memory states and returns the underlying hypothesis (or an approximation of it) with a probability of at least $1/3$ must observe at least $|\cH|^{s'}$ labeled examples.
\end{theorem}
Note that the statement is trivially true if the number of memory states is smaller than $|\cH|$ (since $|\cH|$ memory states are needed to exactly distinguish $|\cH|$ possible hypotheses). 
Thus, the theorem is interesting solely for $a<1/12$.
Recall that most problems with $|\cH|\approx|\cX|$ (and specifically parity) have $\mathrm{d}^2=O(|\cX|)$.
Thus, from the previous theorem we have that for a learning algorithm with (roughly) at most $|\cH|^{1.25}$ memory states, the number of examples needed is exponentially larger than in the case of unbounded memory.

Let us revisit the classes we introduced in the previous section and explore their unlearnability with a bounded-memory learning algorithm. 
We proved that the class $\cH_{th}$ is $\Omega(\sqrt{|\cH_{th}||\cX|})$-mixing and indeed this class can be learned with a bounded-memory algorithm:
on each step it saves a lower and an upper threshold $b_{lower}\leq b_{upper}$ that represent a regime where the correct hypothesis must be located. 
If the algorithm received a labeled example $(x,0)$ with $b_{lower}\leq x\leq b_{upper}$ then $b_{upper}:=x$ else $b_{lower}:=x.$
This algorithm does not use many examples to find an approximation of the underlying hypothesis (the examples are expected to decrease the regime $b_{upper}-b_{lower}$ by some constant factor).
Note that even a generalization of $\cH_{th}$, the class of classification using halfspaces, can be learned with a bounded-memory algorithm, called the Perceptron algorithm \cite{rosenblatt58,shwartz14}.
 
We mentioned that the parity class $\cH_{parity}$ is $O(\sqrt{|\cX|})$-mixing and indeed it was recently discovered that it cannot be learned with even slightly bounded memory \cite{raz16}.
In the last section we proved that most problems are close to random; thus, we can deduce that most problems cannot be learned with memory that is bounded by nearly $|\cH|^{1.25}.$

\subsection{The Inability of a Neural Network to Learn Most Classes}
In Section~\ref{sub_sec:nn_and_ml} we showed that the most frequently used training algorithm implemented for learning a neural network is the SGD which is a bounded-memory algorithm. 
In Section~\ref{sub_sec:nn_and_neuroscience} we explained why any computation made by the brain (i.e., a biologically plausible computation) must be a bounded-memory algorithm. 
From Theorem~\ref{thm-maim} presented in this section we get as a corollary that classes that are $O(\sqrt{|\cX|})$-mixing cannot be learned by neural networks.
From Section~\ref{sec:mixing_graphs}, we also get that most hypothesis classes cannot be learned  by neural networks.

\section{What Can be Learned}\label{sec:when_can_learn}
 In the previous section we explained why most hypothesis classes cannot be learned with a bounded-memory algorithm and specifically by a neural network. 
One might wonder which (and how) classes can be learned with a bounded-memory algorithm. 
Can all the classes that are $\mathrm{d}$-mixing with large $\mathrm{d}$ be learned with a bounded-memory algorithm? 
In Section~\ref{sec:unlearnability_of_mixing} we gave as an example the class $\cH_{th}$ of discrete threshold functions in $[0,1]$ that is $\Omega(\sqrt{|\cH_{th}||\cX|})$-mixing and can easily be learned with a bounded-memory algorithm. 
In this section we consider other natural classes that have \emph{sufficient partitions}  (will be defined formally later, but for now think of them as ``symmetries'' that exist in the class). 
We prove that such classes are $\mathrm{d}$-mixing with large $\mathrm{d}$. 
We then cite evidence (empirically and theoretically) that these problems can be learned with a bounded-memory algorithm.

A convolutional neural network (CNN) is a type of artificial neural network inspired from the animal visual cortex. 
It is a powerful model to solve problems in machine learning and computer vision \cite{krizhevsky12,simonyan14,szegedy15}.
The core idea is to utilize the translation symmetry that exist in images. 
Recently this idea has been generalized to other symmetries \cite{gens14,dieleman16,cohen16}. 

Another form of symmetry was presented in \cite{kadmon16}. 
They considered classes where a small change in the input is considered as noise and thus should be labeled similarly.
They presented a bounded memory algorithm in the form of a small neural network that is able to learn this class. 

More generally we say that a class $\cH$ has an \emph{$r$-sufficient partition} if there is a partition $\cX=\dot\bigcup X_i$ of all the examples into $r$ parts, such that all hypotheses $h\in\cH$ assign the same value to each part (i.e.,  for each $i$ it holds that for all $x\in X_i$, $h(x)$ is equal to each other).
In the next claim we prove that if a class has an $r$-sufficient partition with small $r$, its mixing complexity is small, and thus it is possible that $\cH$ be learned by a bounded-memory algorithm.
\begin{claim}
For any class $\cH$ that has an $r$-sufficient partition, its mixing complexity is bounded by $MC(\cH)=O(\sqrt{r}).$
\end{claim}
\begin{proof}
Since there are $r$ parts in the partition there is at least one part $X$ with size at least $\frac{|\cX|}{r}.$
At least half of the hypotheses $H\subseteq\cH$ (i.e., $|H|\geq|\cH|/2$) either all agree or all disagree with the examples in $X$.
Thus,$$\left|e(H,X)-\frac{|H||X|}{2}\right|\geq \frac{|H||X|}2.$$  
Hence, if the class $\cH$ is $\mathrm{d}$-mixing, then $$\mathrm{2d}\geq\frac{|H||X|}{\sqrt{|H||X|}}=\sqrt{|H||X|}\geq\sqrt{\frac{|\cH|}{2}\frac{|\cX|}{r}}.$$
Thus $\cH$ is $\Omega\left(\sqrt{\frac{|\cH||\cX|}{r}}\right)$-mixing.
\end{proof}

\section{Rethinking Generalization to Understand Deep Learning}\label{sec:google}
In this section we suggest an answer to the open problem presented in \cite{zhang17}: 
\begin{center}
\textit{What distinguishes neural networks that generalize well from those that don't?}
\end{center}

Elegantly \cite{zhang17} put their finger on a tremendous gap in current research on deep learning: the inability to know when the test and training errors are close.
In other words, it is desired to distinguish the classes that have an inherent large generalization error, regardless of the specific learning algorithm used. 
This contrasts with to uniform stability \cite{kearns99,bousquet02} that consider whether a specific learning algorithm does not have a large generalization error.
In other words, uniform stability is a property of an algorithm and not of the hypothesis class. 

These authors \cite{zhang17} illustrated this gap nicely in the following set of experiments. 
They used two known natural image classification datasets (CIFAR10 dataset, \cite{krizhevsky09}, and the ImageNet \cite{russakovsky15} dataset).
They trained a few known neural network architectures using SGD and got small training and test errors; i.e., a small generalization error.  
Then they conducted several experiments, each involving changes in the datasets in some random way. 
In one experiment they changed the labels to be completely random.
In another experiment they used varying levels of label corruptions.
They also experimented with randomly changing the pixels of the image. 
In all of these experiments the test error naturally increased. 
Perhaps surprisingly these state-of-the-art convolutional networks for image classification trained with stochastic gradient methods had a very small training error; i.e., we view a large generalization error. 
They also tried to use different kinds of regularizations (e.g., dropout and weight decay) but concluded that this is unlikely that it is  the fundamental reason for generalization.

They also explained why known complexity measures used in machine learning cannot account for their results.
For example, they pointed out that in the regime they considered the VC-dimension leads to trivial bounds. 
Specifically, they considered neural networks with a number of parameters that was large compared to the sample size. 
Since the VC-dimension of a class with more than $n$ parameters is at least $n$, it does not help to answer their open problem. 
The concept of uniform stability does not help either because it is a property of an algorithm and not of the hypothesis class.
Thus, the conventional wisdom that that tries to answer their open problem by using either properties of the hypothesis class (e.g., VC-dimension), or regularization techniques used during training is flawed. 

Unlike other measures mentioned in \cite{zhang17}, the mixing complexity is able to distinguish between a random hypothesis class and class with natural images. 
As we proved in Section~\ref{sec:mixing_graphs} a random class is $\mathrm{d}$-mixing with small $\mathrm{d}$. 
On the other hand, natural images have some symmetries, and thus not mixing (see Section~\ref{sec:when_can_learn}), and there has been an extensive use of the fact that natural images contain a great deal of structure (e.g., \cite{krizhevsky12,gens14,dieleman16,cohen16})

Note that we cannot use \cite{raz17,moshkovitz17} to justify the use of the mixing complexity in this context. 
In these papers the assumption is that the examples are randomly chosen from $\cX$ on each step.
However, in \cite{zhang17}, there was excessive use of the training data; i.e., similar examples were used multiple times.  
This begs the question as to  whether the proof of \cite{raz17,moshkovitz17} can be generalized to the setting in \cite{zhang17} as well.

\section{Hardness and Robustness of Mixing Classes}\label{sec:perturbation_robustness}
In this section we prove that mixing complexity has several desirable properties. 
First, we prove that any class $\cH$ that is mixing has a VC-dimension $\Omega(\log|\cH|)$. 
Since the VC-dimension of every class $\cH$ is at most $\log_2|\cH|$, we get that such classes are the hardest to learn (up to a constant factor) without memory constraints. 
Furthermore, we prove that mixing complexity is robust under small perturbations.
That is, if a small number of labels is changed, the mixing complexity is approximately unchanged. 

 \subsection{Mixing Hypothesis Classes and VC-dimension}
In this section we focus on hypothesis classes that are mixing and we try to understand the hardness of learning these classes without memory constraints.
As discussed in Section~\ref{sub_sec:preliminaries_VC}, the VC-dimension is used to measure the complexity of learning a class (without memory constraints).
We prove that mixing classes $\cH$ have $VCdim(\cH)=\Omega(\log\cH)$. 
To show this, we will find a set of $k=\Omega(\log\cH)$ examples such that restriction of $\cH$ to these $k$ examples results in $2^k$ different vectors of length $k.$
To find these $k$ samples we note that the first example $x_1$ splits the set of all hypotheses in $\cH$ into two: all hypotheses $h$ with $h(x_1)=1$ all hypotheses $h$ with $h(x_1)=0$. 
The idea of the proof is that for most examples the size of the two sets is almost equal (see Claim~\ref{clm:mixing-sample-partition}). 
The second example $x_2$ splits each of these sets into two again, depending on whether $h(x_2)=1$ or not. 
So in total we have a partition of $\cH$ into four parts, one that contains all hypotheses with $h(x_1)=1$ and $h(x_2)=1$, the second that contains all hypotheses with $h(x_1)=1$ and $h(x_2)=0$, the third that contains all hypotheses with $h(x_1)=0$ and $h(x_2)=1$, and the fourth that contains all hypotheses with $h(x_1)=0$ and $h(x_2)=0$.  
We will prove that for most examples these four parts are almost equal (we apply Claim~\ref{clm:mixing-sample-partition} again). 
Continuing in a similar way for $k=\Omega(\log\cH)$ examples we get that there are $k$ examples that define a partition with $2^k$ parts where none of them is empty.

\begin{claim}\label{clm:mixing-sample-partition}
For any $\mathrm{d}$-mixing graph $(A,B,E)$, any $T\subseteq B$, and any $\epsilon>0$, except for $\frac{2\mathrm{d}^2}{|T|\epsilon^2}$ vertices in $a\in A$, $$\left||\Gamma(a)\cap T|-\frac{|T|}2\right|\leq\epsilon|T|$$
\end{claim}
\begin{proof}
Fix $T\subseteq B.$ Denote by $A_1$ the set of vertices that have too many edges into $T$; i.e., $$A_1=\{a\in A :|\Gamma(a)\cap T|>(1/2+\epsilon)|T|\}.$$ 
Thus, $e(A_1,T)>(1/2+\epsilon)|T||A_1|$. From the mixing property we know that $$e(A_1,T)\leq\frac{|T||A_1|}2 + \mathrm{d}\sqrt{|A_1||T|}.$$
Combining the last two inequalities we get that  $$(1/2+\epsilon)|T||A_1|< \frac{|T||A_1|}2 + \mathrm{d}\sqrt{|T||A_1|}.$$
Thus, $$|A_1|<\frac{\mathrm{d}^2}{|T|\epsilon^2}$$
Similarly we can define the set $A_2\subseteq A$ that have too few edges into $T$; i.e., $A_2=\{a\in A :|\Gamma(a)\cap T|<(1/2-\epsilon)|T|\},$ and prove that this set is small. 
\end{proof}
\begin{corollary}
For any $\mathrm{d}$-mixing hypothesis class $\cH$ its VC-dimension is at least $$VCdim(\cH)=\Omega\left(\min\left\{\log{\frac{|\cH||\cX|}{\mathrm{d}^2}}, \log\cH\right\}\right).$$
\end{corollary}
\begin{proof}
Use Claim~\ref{clm:mixing-sample-partition} with $\epsilon=1/4.$
In each step $i\leq k$, each part will be of size at least $\frac{|\cH|}{4^i}$. 

In each step $i\leq k$ we remove at most $2^{i+1}\frac{\mathrm{d}^24^i}{|\cH|\epsilon^2}$ examples (since there are $2^i$ parts in the partition in each step).
 In total we remove at most $8^{k+3}\frac{\mathrm{d}^2}{|\cH|}$ examples which is smaller than $|\cX|$ for $k+3\leq\log_8{\frac{|\cH||\cX|}{\mathrm{d}^2}}.$
 
Thus, in the last step, $k$, each part is of size at least $1$ for $k\leq \log_4|\cH|$, as we wanted to prove.
\end{proof}

From the previous corollary we can deduce that mixing hypothesis classes are the hardest problems since $VCdim(\cH)=O(\log\cH).$

In \cite{haussler96,langford00} it was suggested that VC-dimension is too crude to be a measure of the number of examples needed to learn.
They showed how to use the shell decomposition method to get better bounds. 
In this method the hypotheses are split according to their test errors.  
Hypotheses with similar test error are considered in the same shell.  
The number of samples needed to learn can be smaller than the one  required by the VC-dimension if the size of the shells is not too large. 
For mixing classes, different hypotheses differ on a substantial number of labeled examples (for the most part), \cite{moshkovitz17}.
Thus for classes that are mixing the size of the shells is large.  
This fact strengthens our understanding that classes that are mixing are indeed the hardest problems to learn.

\subsection{Small Perturbation of Mixing Classes}
In our exploration of the mixing complexity, we would like to know how a small perturbation of a class can change the mixing property. 
Specifically, we would like to know whether a small change to a class that is $\mathrm{d}$-mixing with small $\mathrm{d}$ is a $\mathrm{d}'$-mixing with small $\mathrm{d}'$. 
The next claim answers this question in the affirmative. 

\begin{claim}\label{clm:small_change_still_mixing}
If a hypothesis class $\cH$ is $\mathrm{d}$-mixing then by changing the labels of at most $b$ examples, the resulting class is $\mathrm{d}+\sqrt{b}$ mixing. 
\end{claim}
\begin{proof}
For any $T\subseteq \cH, S\subseteq \cX$ denote by $e'(S,T)$ the number of edges between $S$ and $T$ in the hypotheses graph after the change of $b$ labels. 
Fix $T\subseteq \cH, S\subseteq \cX$ with $|S|=s$ and $|T|=t$. Our goal is to show that $$\left|e'(S,T)-\frac{st}{2} \right|\leq(\mathrm{d}+\sqrt{b})\sqrt{st}.$$
Let us bound the left hand side
\begin{align*}
\left|e'(S,T)-\frac{st}{2} \right| & \leq  \left|e(S,T)-\frac{st}{2}\right| +  \min(b,st)\\
&  \leq \mathrm{d}\sqrt{st} + \min(b,st)\\
&\leq  (\mathrm{d}+\sqrt{b})\sqrt{st} 
\end{align*}
\end{proof}

\section{Conclusions and Open Problems}\label{sec:conclusions}
In this paper we showed the relationships between both artificial and biological neural networks to mixing complexity. 
We showed that it follows from previous papers that if the data is drawn i.i.d from $\cX$, artificial and biological neural network cannot learn most classes.
Empirically it is known that problems of interest do get solved. 
One possible explanation for this apparent contradiction is that the data processed by neural networks in practice might not be i.i.d. 
For example, in many applications of artificial neural networks the same example repeats itself many times. 
Thus, it would be interesting to generalize the results in \cite{raz17,moshkovitz17} to this data acquisition setting. 
Another possible explanation is that problems of interest are not mixing and have a great deal of ``structure'' to them.
We suggested using the notion of $r$-sufficient partitions to formalize the notion of  ``structure''. 
We showed that classes that have such a partition are not mixing. 
It would be interesting to prove that these classes can be learned with a bounded-memory algorithm.

In this paper we also showed that hypothesis classes that are mixing are the hardest learning problems since their VC-dimension is $\Theta(\log|\cH|)$, which is the maximal value possible.
We also showed that these classes are robust in the sense that under a small perturbation of the labels of at most $b$ examples, the mixing complexity can increase by at most $\sqrt{b}$.

\section*{Acknowledgements} 
This work is partially supported by the Gatsby Charitable Foundation, The Israel Science Foundation, and Intel ICRI-CI center. M.M. is grateful to the Harry and Sylvia Hoffman Leadership and Responsibility Program.

\bibliography{mixing_and_nn.bib}
\bibliographystyle{icml2017}

\end{document}